\documentclass[10pt,journal,compsoc]{IEEEtran}

\ifCLASSOPTIONcompsoc
  \usepackage[nocompress]{cite}
\else
  \usepackage{cite}
\fi

\usepackage{amsmath}
\usepackage{amsfonts}
\usepackage{amssymb}
\usepackage{bm}
\usepackage{mathrsfs}
\usepackage{mathtools}
\usepackage{amsthm}
\usepackage{algorithm}
\usepackage{algorithmic}

\usepackage{array}
\usepackage{booktabs}
\usepackage{multirow}
\usepackage{makecell}

\usepackage{graphicx}
\usepackage[caption=false,font=normalsize,labelfont=sf,textfont=sf]{subfig}
\usepackage{stfloats}
\usepackage{textcomp}

\usepackage[dvipsnames,table]{xcolor}
\usepackage{colortbl}

\usepackage{pifont}
\usepackage{xspace}

\usepackage{url}

\usepackage[hidelinks]{hyperref}

\newcommand{\sys}{CrystalMem\xspace}

\newtheorem{theorem}{Theorem}

\newtheorem{proposition}[theorem]{Proposition}

\newtheorem{definition}[theorem]{Definition}
\newtheorem{assumption}{Assumption}
\theoremstyle{remark}

\newtheorem{observation}{Observation}

\colorlet{phaseI}{rgb:red!2,65;green!30,60;blue!20,125}
\colorlet{phaseII}{rgb:red!2,65;green!30,90;blue!20,125}
\colorlet{phaseIII}{rgb:red!60,100;green!20,90;blue!30,125}

\newcommand{\Comment}[1]{\hfill$\triangleright$ \textit{#1}}

\newcommand{\best}[1]{\textcolor{red}{\textbf{#1}}}

\newcommand{\venue}[1]{{\scriptsize\color{gray}[#1]}}
\newcommand{\std}[1]{$_{\pm\text{\scriptsize #1}}$}
\definecolor{deepred}{rgb}{0.6,0,0}
\definecolor{crystalblue}{rgb}{0.20,0.40,0.80}

\definecolor{darkgreen}{RGB}{0, 181, 18}
\definecolor{darkred}{RGB}{252, 90, 90}
\newcommand{\gain}[1]{$\uparrow$ #1}
\newcommand{\mygreen}[1]{\cellcolor{darkgreen!#1}}

\definecolor{DeltaBg}{HTML}{D4F2D7}
\definecolor{SearchBg}{HTML}{C2E6F5}
\definecolor{AgenticBg}{HTML}{F5C2CC}
\definecolor{MathBg}{HTML}{E6D4F2}
\definecolor{ScienceBg}{HTML}{FBE0BC}

\begin{document}

\title{CrystalMem: Elastic Memory for Self-Evolving\\ LLM Agents via Knowledge Crystallization}

\author{Beining~Wu,~\IEEEmembership{Member,~IEEE,}
and Jun~Huang,~\IEEEmembership{Senior~Member,~IEEE}%
\IEEEcompsocitemizethanks{\IEEEcompsocthanksitem
Beining Wu and Jun Huang are with the Department of Electrical Engineering and
Computer Science, South Dakota State University, Brookings, SD 57006 USA.\protect\\
E-mail: Wu.Beining@jacks.sdstate.edu; Jun.Huang@sdstate.edu}}

\markboth{Preprint}%
{Wu and Huang: CrystalMem}

\IEEEtitleabstractindextext{%
\begin{abstract}
Memory for self-evolving large language model (LLM) agents is often
provisioned as if its byte budget only grows. Cloud platforms, however,
adjust quotas with load and cost, and we show that capability does not
follow the budget back up: after a squeeze-and-recover cycle, the agent
settles below its pre-squeeze level, a gap we call \emph{memory
hysteresis}. The cause is structural. Deletion and one-way compression
discard the material needed for later rebuilding, and we prove that any
policy that only keeps or drops entries carries a residual-deficit floor.
We propose \sys{} (\underline{\textbf{Crystal}}lized
\underline{\textbf{Mem}}ory), an elastic memory sidecar that demotes
entries across four fidelity states under a crystallization-energy
schedule, orders demotions by advantage-weighted influence with dependency
coupling, and recovers capability through verified recrystallization under
explicit compute and byte caps. Across seven environments, seventeen
methods, and six backbones, with multi-tenant serving and a physical
edge--cloud deployment, \sys{} achieves the highest restored capability in
every setting and closes the loop left open by every baseline. From a
$50\%$ byte budget, \sys{} matches the strongest budgeted baseline at full
provision on every environment; at equal budgets, it leads by $+4.6$\,pp
on average.
\end{abstract}

}

\maketitle

\IEEEdisplaynontitleabstractindextext
\IEEEpeerreviewmaketitle

\section{Introduction}
\label{sec:intro}

\IEEEPARstart{S}{elf}-evolving large language model (LLM) agents improve by
accumulating experience. Trajectories, distilled skills, and facts are
written to an external memory and retrieved for later
tasks~\cite{Du2026ARXIV,Wu2026TNSE,Wang2023ARXIV,Ouyang2026ICLR,Fang2025ARXIV,Wu2026ARXIV1,Ding2026TAAS}.
Production
systems now treat this store as a managed memory plane, with consolidation
pipelines~\cite{Packer2024COLM,Chhikara2025ARXIV,Kang2025EMNLP} and
value-based governance~\cite{Wu2026ARXIV,Yu2026ARXIV,Ding2026ARXIVEASE,Ding2026MNET,Ding2026ICDCS}
deciding what stays.
One layer below, the cloud serving stack already treats resources as
elastic: compute is
scheduled~\cite{Du2026TCC,Zhang2026TCC,Fang2026ARXIVLLMSearch,Duan2023COMST,Huang2025TMC,Ding2026ARXIVTwinLoop,Wu2026TON},
the key--value
cache is paged and reclaimed~\cite{Kwon2023SOSP}, and quotas move with load
and cost. Agent memory is the exception. Its byte budget is usually
provisioned as if it could only grow, although a shared cloud offers no
such guarantee.

Budgets move in every other tier, but prior work has not asked what a
round trip of the memory budget does to agent capability. Our measurements
show that capability does not simply return with the bytes. When the quota
falls, current stores shed bytes by deleting entries or by compressing
them in one direction; both choices discard material that would be needed
to rebuild the entry later. When the quota returns, nothing regenerates.
The agent settles below its pre-squeeze level even though the original
provision is fully available again (Figs.~\ref{fig:recovery-box}
and~\ref{fig:damage-bars}). We call this gap \emph{memory hysteresis}: the
store's capability depends not only on the budget it has, but also on the
path the budget took. Retention therefore has to become reversible.
Compression ladders that demote entries into coarser
representations~\cite{Kerestecioglu2026ARXIV} keep more information than
deletion, but they only move downward, so each squeeze leaves a permanent
loss. Reversible compression exists~\cite{Wang2025ACL}, yet reversibility
alone does not decide which entry to demote under a falling budget or
whether a later reconstruction is faithful. Value-based governance chooses
better survivors~\cite{Wu2026ARXIV,Wu2026ICDCS,Alla2025ARXIV,Pudasaini2026HPSR}, but a policy that keeps
an entry verbatim or drops it leaves no residue to rebuild from. Closing
the loop requires fidelity states to demote into, a schedule that moves
entries in both directions as the budget moves, and a gate that verifies
what returns.

These requirements sit across three research lines. Agent memory and
self-evolution build stores that grow, consolidate, and retrieve under a
fixed provision~\cite{Xu2025NEURIPS,Packer2024COLM,Chhikara2025ARXIV,Kang2025EMNLP,Zhou2026ICLR}.
Memory compression and governance reduce or curate the store through
one-way ladders, opportunistic reversibility, or binary
selection~\cite{Alqithami2025ARXIV,Kerestecioglu2026ARXIV,Wang2025ACL,Dai2026ACL,Wu2026ARXIV,Zhang2026ICLRW}.
Elastic serving and evaluation reclaim resources below the memory plane
and measure memory at fixed
provisions~\cite{Kwon2023SOSP,Lin2025ARXIV,Du2026TCC,Zhang2026TCC,Margalit2026ARXIV,Hu2026ICLR,Kwon2026ARXIV,Wu2026COMST,Wu2023ACCESS,Wu2026MNET,Fang2025TON,Dong2026TCCN,Xing2025ACR,Pan2023SCIS}.
None treats fidelity as a two-way function of the byte budget, schedules
demotion and promotion against that budget, or verifies what a promotion
brings back. As a result, no current system closes the capability loop
under an elastic memory budget.

To close this loop, we introduce \sys{}
(\underline{\textbf{Crystal}}lized \underline{\textbf{Mem}}ory), an
elastic memory sidecar for self-evolving agents. \sys{} treats retention
as reversible state assignment through two coupled mechanisms. First,
\emph{knowledge crystallization} assigns every entry to one of four
fidelity states, from full form down to a minimal trace. A budget-driven
schedule prices each demotion by the utility it forgoes now plus the
compute a later promotion will cost, then takes the cheapest moves first
until the store fits the incoming budget. Entry values come from an
influence estimator that attributes task outcomes to the entries that
served them. Value also flows over a dependency graph, so entries that
anchor other entries sink last. Second, \emph{verified recrystallization}
keeps a residue and a set of invariants for every demoted entry. When the
budget expands, the store regenerates entries from their residues and
admits only reconstructions that pass the invariant check, under an
explicit cap on regeneration compute. Both passes run inside the serving
path with one sorted sweep per stage and no training. Deletion becomes the
last resort instead of the default.

In summary, our main contributions are as follows. They divide into
findings and designs: the phenomenon and its floor are findings; the
crystallization mechanisms are designs; the serving stack and benchmark
environments are adopted from prior work and are not claimed as
contributions.

\begin{itemize}
\item We identify and measure \emph{memory hysteresis} in elastic-budget
agent memory: a squeeze-and-recover cycle leaves capability below its
pre-squeeze level even after the budget fully returns. We formalize the
effect with a loop-area metric and prove a residual-deficit floor for
keep-or-drop policies.

\item We propose \sys{}: a four-state fidelity ladder scheduled by
crystallization energy under the live budget, entry values estimated from
advantage-weighted influence with dependency coupling, and verified
recrystallization under explicit compute and byte caps. We prove
feasibility, near-optimality, and loop-closure guarantees.

\item We evaluate \sys{} against seventeen methods across seven
environments and six backbones, in multi-tenant serving, and on a
physical edge--cloud testbed. \sys{} closes the loop left open by every
baseline, matches full-provision capability at half the byte budget, and
recovers capability at a fraction of the energy needed to re-experience
the underlying episodes.
\end{itemize}

The rest of this paper is organized as follows.
Section~\ref{sec:related} reviews related work.
Section~\ref{sec:hysteresis} formalizes elastic memory budgets and
measures memory hysteresis. Section~\ref{sec:design} presents \sys{}.
Section~\ref{sec:experiments} reports the experimental evaluation.
Section~\ref{sec:conclusion} concludes the paper.

\section{Related Work}
\label{sec:related}

\subsection{Agent Memory and Self-Evolution}

A self-evolving large language model (LLM) agent improves without weight
updates by writing experience into external memory and retrieving it at
inference time~\cite{Du2026ARXIV,Wu2026TNSE,Wu2026ARXIVPRISM,Wu2026ARXIVRELIEF,Zhang2026ARXIVHarness}.
One line of work builds the store itself.
Xu \textit{et al.}~\cite{Xu2025NEURIPS} propose \emph{A-MEM}, which links
and grows notes for long-horizon recall; \emph{MemGPT}~\cite{Packer2024COLM}
pages entries between the context window and external storage;
\emph{Mem0}~\cite{Chhikara2025ARXIV} extracts and consolidates long-term
memory at production scale; and \emph{MemoryOS}~\cite{Kang2025EMNLP} brings
operating-system structure to storage, update, and retrieval. These systems
make memory larger, better organized, and easier to retrieve.

A second line turns experience into competence. \emph{Voyager}~\cite{Wang2023ARXIV}
grows a skill library from open-ended exploration; \emph{Agent Workflow
Memory}~\cite{Wang2024ARXIV} and \emph{ReasoningBank}~\cite{Ouyang2026ICLR}
distill trajectories into reusable workflows and strategies; \emph{EvolveR}~\cite{Wu2026ICML}
and \emph{PolySkill}~\cite{Yu2026ICLR} close an experience-driven lifecycle
over skills; \emph{AutoRefine}~\cite{Qiu2026ARXIV} keeps stored expertise
current; and \emph{MEM1}~\cite{Zhou2026ICLR} and
\emph{Memory-as-Action}~\cite{Zhang2026ACL} train the agent to curate its
own context. Across both lines, the store grows and is trimmed under a fixed
provision. The byte budget is treated as a background constraint rather than
a control variable, and neither line asks how capability changes when the
budget shrinks and then returns.

\subsection{Memory Compression and Governance}

When memory must shrink, one family compresses it. \emph{MaRS}~\cite{Alqithami2025ARXIV}
summarizes episodes into reflections at a single tier; \emph{RecMem}~\cite{Dai2026ACL}
consolidates recurrently for long-running agents; \emph{Human-Inspired}
memory~\cite{Kerestecioglu2026ARXIV} demotes entries through a ladder of
representations, but only downward; and \emph{R$^3$Mem}~\cite{Wang2025ACL}
makes compression reversible, yet promotes opportunistically, with no
budget-driven schedule and no gate that verifies what returns. A parallel
theory line frames memory compaction as a rate--distortion trade-off over a
single compression event~\cite{Zou2026ARXIV,Colaco2026ARXIV,Zhang2026ARXIVb}.
In these systems, fidelity transitions are one-way or unscheduled; none treats
them as a two-way function of a moving byte budget.

Closest to ours, a governance family decides which entries deserve the bytes.
\emph{CURATOR}~\cite{Wu2026ARXIV} scores net value per byte and governs
retention under a hard footprint; \emph{A-MAC}~\cite{Zhang2026ICLRW} learns
what to admit at write time; \emph{AgeMem}~\cite{Yu2026ARXIV} learns a unified
long- and short-term management policy; \emph{Darwinian} memory~\cite{Mi2026ICML}
lets a training-free store regulate itself; and \emph{BudgetMem}~\cite{Alla2025ARXIV}
learns selective retention under explicit cost. These policies identify
valuable entries, and the strongest of them is our binary baseline. However,
each policy either keeps an entry verbatim or drops it. Selection decides which
entries survive a squeeze, but not what fidelity the survivors should occupy
or how evicted capability should return when the budget does.

\subsection{Elastic Serving and Evaluation}

Elasticity is routine in the serving stack below the memory
store~\cite{Fang2025JSAC,Fang2026GLOBECOM,Wu2025WASA,Wu2025RACS,Ding2025IPCCC,Ding2026ICNC,Wu2023MPE}.
vLLM~\cite{Kwon2023SOSP} manages the key--value cache as a paged, reclaimable
resource; sleep-time compute~\cite{Lin2025ARXIV} spends idle cycles preparing
context before queries arrive; and governed shared memory~\cite{Margalit2026ARXIV}
arbitrates a store written by many agents. These systems reclaim serving memory
or schedule maintenance compute, but they do not tie the fidelity of an
experiential store to the byte budget it receives. \sys{} composes with this
layer rather than replacing it.

Evaluation is beginning to measure lifecycle effects. \emph{MemoryAgentBench}~\cite{Hu2026ICLR}
measures accumulation and selective forgetting over incremental interaction;
\emph{Reclaim Eval}~\cite{Kwon2026ARXIV} finds that lossy retention can leave
an agent worse off than no memory at all; and \emph{Neuromem}~\cite{Zhang2026ARXIVa}
decomposes the streaming lifecycle of an external store into stages. Each
protocol measures at a fixed provision or at a single degradation point, so
capability that fails to return with the budget remains invisible. To our
knowledge, no prior system schedules fidelity movement in both directions under
a byte budget, and no prior protocol measures what an elastic budget does to
agent capability.

\section{Preliminaries and Memory Hysteresis Analysis}
\label{sec:hysteresis}

This section names and formalizes the failure mode: under an elastic budget,
capability lost during a squeeze does not automatically return when the
budget recovers. We model the memory plane and its elastic budgets
(Sec.~\ref{sec:model}), define and measure memory hysteresis
(Sec.~\ref{sec:phenomenon}), and show why deletion-based retention cannot
close the loop (Sec.~\ref{sec:impossibility}).

\subsection{System Model and Elastic Budgets}
\label{sec:model}

We consider a self-evolving LLM agent hosted in the cloud. Completed
episodes append experiential entries, such as trajectories, distilled
skills, and environment facts, to a memory store. At inference time, the
agent retrieves from this store to condition its policy. The store is a
metered resource, and its provisioned budget changes with cost control,
tenant load, and pricing events. We call this regime an \emph{elastic
budget}, and we call the storage side of the agent the memory plane.

\begin{table}[t]
\centering
\caption{Notation.}
\label{tab:notation}
\begin{tabular}{ll}
\toprule
Symbol & Meaning \\
\midrule
$m$, $\mathcal{D}_m$ & environment index; its query distribution \\
$q$, $g(q,M)$ & query; graded outcome in $[0,1]$ under store $M$ \\
$e$, $\mu(\cdot)$ & memory entry; byte measure \\
$M^s$, $W^s$ & store at stage $s$; entries written during stage $s$ \\
$b^s$, $B_0$ & budget fraction at stage $s$; provisioned bytes \\
$S$, $\Delta_b$ & number of stages; budget span $\max_s b^s-\min_s b^s$ \\
$\pi$, $\Pi_b$ & memory policy; binary-retention class \\
$C_m^s$ & capability on environment $m$ at stage $s$ \\
$\Gamma$, $H_m$ & capability--budget loop; hysteresis area \\
$D_m$ & residual deficit after the budget returns \\
$\bar{M}$, $R$ & entries evicted in the squeeze; re-observed subset \\
$u_e$, $p_e$ & terminal marginal utility; re-observation bound \\
\bottomrule
\end{tabular}
\end{table}

Let $e$ denote an entry and $\mu(e)$ its serialized size in bytes. For a
store $M$, write $\mu(M)=\sum_{e\in M}\mu(e)$. The agent serves an
environment $m$ with query distribution $\mathcal{D}_m$. Time is divided
into stages $s=1,\dots,S$. During stage $s$, the agent holds store $M^s$,
answers queries, and appends new writes $W^s$. A \emph{memory policy} $\pi$
then maps $(M^s, W^s)$ to the next store $M^{s+1}$ under the byte budget of
stage $s{+}1$. Capability is the expected graded outcome
$C_m^s=\mathbb{E}_{q\sim\mathcal{D}_m}\!\left[g(q,M^s)\right]$. The
provider-facing problem is to sustain this capability across the full budget
path:
\begin{equation}
\label{eq:problem}
\max_{\pi}\;\frac{1}{S}\sum_{s=1}^{S}
\mathbb{E}_{q\sim\mathcal{D}_m}\!\left[g\!\left(q,M^s\right)\right]
\quad\mathrm{s.t.}\;\; \mu\!\left(M^s\right)\le b^s B_0\;\;\forall s.
\end{equation}

Deployed policies fall into two classes. \emph{Binary retention} $\Pi_b$
keeps an entry verbatim or drops it, so $M^{s+1}\subseteq M^s\cup W^s$.
Recency eviction, expiry, and learned value governance~\cite{Wu2026ARXIV}
belong to this class. \emph{Irreversible ladders} may replace an entry by a
lossy transform, but transitions run in only one direction: once demoted, an
entry is never rebuilt~\cite{Kerestecioglu2026ARXIV}. Both classes share the
same structural property. A squeeze step is a projection, and what it
discards re-enters the store only if the input stream supplies it again.

\begin{definition}[Elastic budget schedule]
\label{def:schedule}
An elastic cycle is the budget path
$b=(b^1,\dots,b^S)=(1,\,0.75,\,0.5,\,0.25,\,0.5,\,0.75,\,1)$ with $S=7$:
a \emph{squeeze} phase $s\in\{1,\dots,4\}$ and a \emph{recovery} phase
$s\in\{4,\dots,7\}$, with span $\Delta_b=0.75$. All policies run under the
same path and the same byte accounting (iso-budget).
\end{definition}

\subsection{The Memory Hysteresis Phenomenon}
\label{sec:phenomenon}

Elastic budgets turn memory management into a round trip: the store is
squeezed to a quarter of its provision and then restored. If the squeeze were
reversible, capability would retrace the same path during recovery. We borrow
\emph{hysteresis} from magnetism for the failure of this retracing: the state
of the system depends on the path of the control variable, not only on its
current value. The cycle traces a closed path
$\Gamma=\{(b^s,C_m^s)\}_{s=1}^{S}$ in the capability--budget plane, and the
enclosed area measures how much recovery lags the squeeze.

\begin{figure}[t]
\centering
\includegraphics[width=0.9\columnwidth]{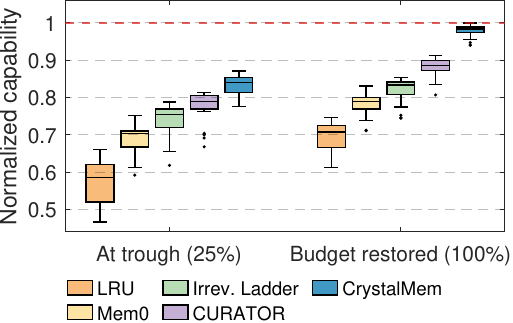}
\caption{\textbf{Memory hysteresis under an elastic budget cycle.} Capability
normalized to each method's pre-squeeze level (seven environments $\times$
five seeds). After the budget returns to 100\%, every baseline stays below
its pre-squeeze capability; \sys\ returns to the full-recovery line (red
dashed).}
\label{fig:recovery-box}
\end{figure}

\begin{definition}[Hysteresis area and residual deficit]
\label{def:hysteresis}
For a policy $\pi$ on environment $m$ under Definition~\ref{def:schedule},
the \emph{hysteresis area} is the normalized loop integral
\begin{equation}
\label{eq:area}
H_m(\pi)\;=\;\frac{1}{C_m^1\,\Delta_b}
\left|\;\oint_{\Gamma} C\,\mathrm{d}b\;\right|,
\end{equation}
estimated from the staged measurements by the trapezoid sum
\begin{equation}
\label{eq:est}
\hat{H}_m\;=\;\frac{1}{C_m^1\,\Delta_b}
\left|\;\sum_{s=1}^{S-1}\tfrac{1}{2}
\left(C_m^{s+1}+C_m^{s}\right)\left(b^{s+1}-b^{s}\right)\right|,
\end{equation}
and the \emph{residual deficit} is $D_m(\pi)=C_m^1-C_m^S$, the capability
still missing once the budget is fully restored.
\end{definition}

$H_m=0$ exactly when the cycle retraces itself. Both quantities are
normalized by the pre-squeeze capability $C_m^1$, making them comparable
across environments.

We measure one representative from each policy family under
Definition~\ref{def:schedule}: recency eviction (LRU), single-tier
summarization (Mem0~\cite{Chhikara2025ARXIV}), a four-level irreversible
ladder~\cite{Kerestecioglu2026ARXIV}, binary value governance
(CURATOR~\cite{Wu2026ARXIV}), and \sys. The testbed covers seven
environments: retrieval (SF, TTL~\cite{Hu2026ICLR}), long-context dialogue
(LoCoMo~\cite{Maharana2024ACL}, LongMemEval~\cite{Wu2025ICLR}), agentic
control (ALFWorld~\cite{Shridhar2021ICLR}, WebArena~\cite{Zhou2024ICLR}),
and synthetic drift (SynDrift), with a Qwen2.5-7B host~\cite{Yang2024ARXIV},
a bge-m3 retriever~\cite{Chen2024ACL}, and five seeds. Two regularities
emerge.

\begin{observation}[Ordered degradation at the trough]
\label{obs:trough}
At the squeeze trough ($b^4=0.25$), median capability retention orders by
policy family: $0.585$ (LRU) $<0.703$ (Mem0) $<0.755$ (irreversible ladder)
$<0.790$ (CURATOR) $<0.841$ (\sys); see
Fig.~\ref{fig:recovery-box}, left group.
\end{observation}

\begin{observation}[Restoration does not restore]
\label{obs:open}
After the budget returns to $100\%$, the four baselines plateau at
$0.71$--$0.89$ of their pre-squeeze capability and trace open loops with
$\hat{H}_m$ between $0.086$ and $0.207$; \sys\ returns to $0.984$ with
$\hat{H}_m=0.010$. See Fig.~\ref{fig:recovery-box}, right group, and
Fig.~\ref{fig:damage-bars}.
\end{observation}

The mechanism is simple. Evicted entries are gone, recovery-phase writes
replace only what the stream supplies again, and retrieval degrades around
the missing entries. The next subsection shows that, for deletion-based
retention, this behavior follows from a bound rather than an implementation
choice.

\subsection{Why Existing Memory Policies Cannot Close the Loop}
\label{sec:impossibility}

\begin{figure}[t]
\centering
\includegraphics[width=0.9\columnwidth]{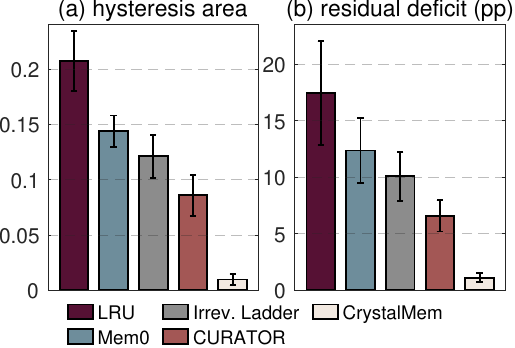}
\caption{\textbf{Two currencies of hysteresis damage}: (a) normalized
hysteresis-loop area and (b) residual capability deficit after budget
restoration (mean$\pm$std across seven environments).}
\label{fig:damage-bars}
\end{figure}

Two assumptions isolate what the argument needs.

\begin{assumption}[Bounded re-observation]
\label{asm:reobs}
Let $\bar{M}$ be the set of entries evicted during the squeeze phase. Each
$e\in\bar{M}$ is supplied again by the input stream during the recovery
phase with probability at most $p_e<1$.
\end{assumption}

Episodic streams rarely repeat: conversations move on, tasks change, and
drifting environments retire facts. Thus $p_e$ is small in practice.

\begin{assumption}[Additive terminal utility]
\label{asm:util}
There exist $u_e\ge 0$ such that removing any $A\subseteq M$ from a store
$M$ at the terminal stage lowers capability by at least
$\sum_{e\in A}u_e$; redundancy is absorbed into the definition of $u_e$.
\end{assumption}

\begin{proposition}[Deficit lower bound for binary retention]
\label{prop:bound}
Fix the schedule of Definition~\ref{def:schedule} and any
$\pi\in\Pi_b$. Let $\tilde{C}$ denote the terminal capability of the same
run with $b^s\equiv 1$. Under
Assumptions~\ref{asm:reobs}--\ref{asm:util},
\begin{equation}
\label{eq:bound}
\mathbb{E}\!\left[\tilde{C}-C_m^S\right]\;\ge\;
\sum_{e\in\bar{M}} u_e\left(1-p_e\right).
\end{equation}
\end{proposition}

\begin{proof}
Fix the stream and condition on the squeeze. Let $R\subseteq\bar{M}$ be the
evicted entries supplied again during the recovery phase;
Assumption~\ref{asm:reobs} gives $\mathrm{P}(e\in R)\le p_e$. Because $\pi$
is binary, an eviction leaves no residue. The terminal store is therefore
$M^S=\left(M^1\setminus\bar{M}\right)\cup R\cup W$, where $W$ collects
recovery-phase writes, while the un-squeezed counterfactual holds
$\tilde{M}=M^1\cup W$. Hence
$\tilde{M}\setminus M^S\supseteq\bar{M}\setminus R$, and
Assumption~\ref{asm:util} applied to $A=\bar{M}\setminus R$ yields
$\tilde{C}-C_m^S\ge\sum_{e\in\bar{M}\setminus R}u_e$. Taking expectation
over re-observation gives
$\mathbb{E}\!\left[\tilde{C}-C_m^S\right]
\ge\sum_{e\in\bar{M}}u_e\,\mathrm{P}(e\notin R)
\ge\sum_{e\in\bar{M}}u_e\left(1-p_e\right)$.
\end{proof}

The bound is not vacuous: the trough expels at least $1-b^4=75\%$ of
resident bytes, so the sum in \eqref{eq:bound} covers most of what the agent
has learned. We report deficits against the pre-squeeze anchor $C_m^1$.
Value governance operates within the bound, not beyond it. Choosing which
entries to evict reduces the sum but cannot make it vanish, and the
strongest binary baseline still forfeits $6.6\pm1.4$\,pp, compared with
$17.5\pm4.6$\,pp for recency eviction (Fig.~\ref{fig:damage-bars}(b)).

Lossy tiers relax the binary premise only partially. A single tier fixes one
exchange rate between bytes and fidelity; a $4\times$ squeeze exceeds that
rate, so the tier deletes what it cannot hold and inherits the bound
($12.4\pm2.9$\,pp for Mem0). A one-way ladder keeps residues at several
rates but never promotes them, so it serves degraded representations even
after the budget returns ($10.1\pm2.2$\,pp). The measurements separate the
two levers: four tiers instead of one buy $2.3$\,pp, while the reversible
stack of \sys\ buys another $9.0$\,pp ($1.1\pm0.4$\,pp). The proof uses only
the absence of a residue from which an entry can be rebuilt. Store such a
residue, and promotion by compute rather than re-observation becomes
possible. Reversibility is the first-order lever; tier count is second-order.

The analysis determines the design. The store must hold entries at several
fidelities and move them in both directions, using a reversible ladder rather
than one-way demotion. Because the sum in \eqref{eq:bound} is weighted by
$u_e$, demotion must track marginal utility, and entries that anchor others
must sink last. Because re-synthesis spends tokens and can hallucinate,
promotion must be budgeted and verified before it re-enters the store.
\sys\ is built around exactly these requirements.

\section{CrystalMem Design}
\label{sec:design}

\subsection{Overview}
\label{sec:overview}

The analysis gives three design requirements: a reversible ladder,
utility-aware demotion, and verified promotion. \sys{} implements them with a
single control loop around the store (Fig.~\ref{fig:framework}), executed once
per stage in three phases. \emph{Monitor} maintains per-entry value estimates.
\emph{Crystallize} reassigns fidelity states to satisfy the incoming byte
budget. \emph{Recrystallize} buys back capability when the budget expands. The
stage transition factorizes as
\begin{equation}
\label{eq:transition}
\begin{aligned}
p\!\left(M^{s+1}\mid M^{s}, b^{s+1}\right)=
&\Bigg[\prod_{e\in M^{s}}\varphi\!\left(f_e^{s+1}\mid f_e^{s},\eta_e\right)\Bigg]\\
\times\,&\Bigg[\prod_{e\in U^{s}}\mathcal{R}\!\left(e'\mid z_e, M^{s}\right)
\mathcal{V}\!\left(v_e\mid e', h_e\right)\Bigg],
\end{aligned}
\end{equation}
where the first factor crystallizes: an assignment kernel $\varphi$ moves each
entry along the fidelity ladder $f_e$ under a crystallization energy $\eta_e$.
The second factor recrystallizes: a generation operator $\mathcal{R}$ rebuilds
promoted entries from their residues $z_e$, and a verification operator
$\mathcal{V}$ checks the result against stored invariants $h_e$. Deletion does
not appear in \eqref{eq:transition}. It remains only as last-resort eviction of
bottom-rung residues.

\begin{figure*}[t]
\centering
\includegraphics[width=0.95\textwidth]{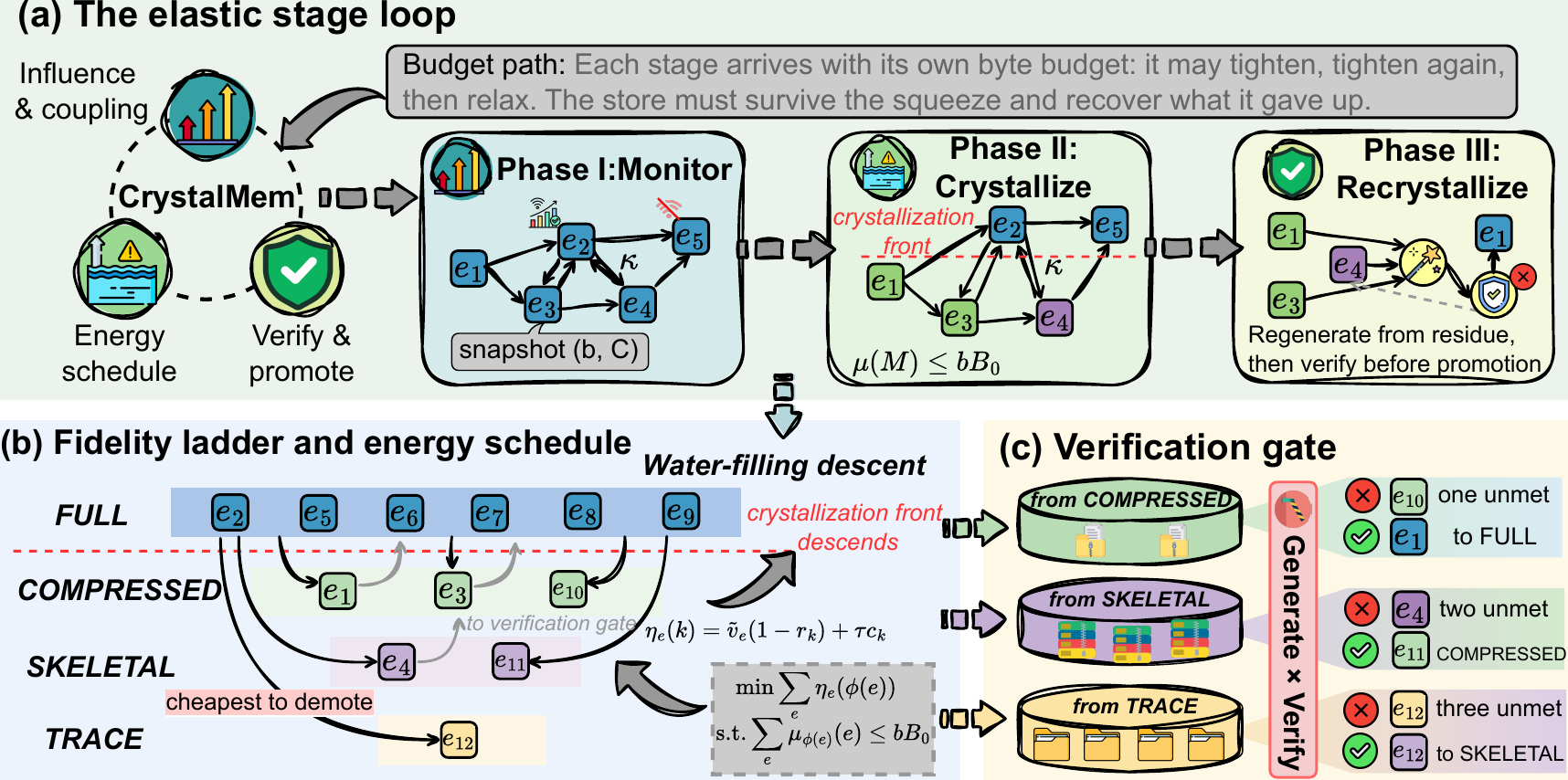}
\caption{\textbf{\sys\ framework.} (a)~Each stage runs one control loop:
\emph{Monitor} scores entries by advantage-weighted influence with dependency
coupling, \emph{Crystallize} fits the store to the incoming byte budget, and
\emph{Recrystallize} buys capability back when the budget relaxes. (b)~A
crystallization front descends the four-state fidelity ladder by water-filling
on energy, demoting the cheapest entries first. (c)~Recrystallization
regenerates promoted entries from their residues and admits a draft only after
verification confirms the invariants stored at demotion time.}
\label{fig:framework}
\end{figure*}

\subsection{Fidelity Ladder and Crystallization Energy}
\label{sec:ladder}

Observation~\ref{obs:open} traced open loops to one-way demotion. The ladder
replaces deletion with reversible state assignment.

Every entry occupies one of $K=4$ states, $f_1$ (\textsc{full}) through $f_4$
(\textsc{trace}), with \textsc{compressed} and \textsc{skeletal} between. At
state $k$, an entry costs $\mu_k(e)$ bytes and delivers a fraction $r_k$ of its
utility, with $\mu_1(e)>\dots>\mu_4(e)$ and $1=r_1>\dots>r_4$; regenerating
from state $k$ costs $c_k$ tokens, increasing in $k$. From \textsc{compressed}
down, each entry keeps a residue $z_e$, namely its state-$k$ representation,
and a set of invariants $h_e$: key spans, entity tuples, and outcome facts
distilled at demotion time. These invariants later anchor verification.

A demotion decision trades utility now against compute later. We price both
terms in one score,
\begin{equation}
\label{eq:energy}
\eta_e(k)\;=\;\tilde{v}_e\left(1-r_k\right)\;+\;\tau\, c_k ,
\end{equation}
the utility forgone by holding $e$ at state $k$ plus the anticipated cost of
buying it back, with $\tau$ converting regeneration tokens into capability
cost. Given the incoming budget, Crystallize solves the constrained assignment
\begin{equation}
\label{eq:schedule}
\begin{aligned}
\varphi^{\star}=\operatorname*{arg\,min}_{\varphi}\;
&\sum_{e\in M^{s}}\eta_e\!\left(\varphi(e)\right)\\
\mathrm{s.t.}\;\;
&\sum_{e\in M^{s}}\mu_{\varphi(e)}(e)\;\le\; b^{s+1}B_0 .
\end{aligned}
\end{equation}
The objective is separable, so the Lagrangian
$\sum_e \eta_e(\varphi(e))+\theta\,(\sum_e\mu_{\varphi(e)}(e)-b^{s+1}B_0)$
is minimized rung by rung: an entry stays at the deepest state whose marginal
price $\Delta\eta_e/\Delta\mu_e$ remains below the multiplier $\theta$ set by
the budget. \sys{} implements this step as a water-filling descent, demoting
one rung at a time in increasing order of marginal price until the store fits.
We call the moving boundary the crystallization front
(Fig.~\ref{fig:framework}(b)).

\subsection{Influence-Driven Assignment}
\label{sec:influence}

The bound \eqref{eq:bound} charges a policy for every unit of deleted utility,
so demotion order must track $u_e$. \sys{} estimates this value online.

Each retrieval attributes the query outcome back to the entries it surfaced,
baselined against recent outcomes:
\begin{equation}
\label{eq:influence}
I_e\;\leftarrow\;\gamma\, I_e+\left(1-\gamma\right) w_e\,
\frac{g\!\left(q,M\right)-\operatorname*{mean}_{j\in Q}\, g_j}
{\operatorname*{std}_{j\in Q}\, g_j}\,,
\end{equation}
where $w_e$ is the normalized retrieval weight of $e$ in answering $q$, $Q$ is
a sliding batch of recent queries, and $\gamma$ is the averaging horizon.
Entries that repeatedly appear in above-average outcomes accumulate influence;
stale or decorative entries decay toward zero. The estimator is white-box: it
uses only retrieval logs and graded outcomes, both already produced by the
serving path.

Entries are not independent: a skill cites facts, and a trajectory can ground a
distilled rule. \sys{} maintains a dependency graph over co-retrieval and
reference edges and lets influence flow across it,
\begin{equation}
\label{eq:coupling}
\tilde{v}_e\;=\;I_e+\lambda\!\sum_{e'\in N(e)}
\frac{x\!\left(e,e'\right)}{\sum_{j\in N(e)} x\!\left(e,j\right)}\,
I_{e'} ,
\end{equation}
where $N(e)$ is the neighborhood of $e$, $x(e,e')$ counts co-retrievals, and
$\lambda$ sets the coupling strength. Hub entries take value from the entries
they anchor and therefore sink last, matching the demotion order required by
the analysis. As a hard guard, an entry is not demoted below \textsc{skeletal}
while a \textsc{full} neighbor still cites its invariants.

\subsection{Recrystallization}
\label{sec:recrystallize}

Proposition~\ref{prop:bound} depends on evictions leaving no residue.
Recrystallization removes that condition.

When the budget expands, Recrystallize selects a promotion set
\begin{equation}
\label{eq:promote}
\begin{aligned}
U^{s}=\operatorname*{arg\,max}_{U\subseteq M^{s}}\;
&\sum_{e\in U}\tilde{v}_e\,\Delta_e\\
\mathrm{s.t.}\;\;\sum_{e\in U} c_e\le\rho\, n^{s},
\;\;&\mu\!\left(M^{s}\right)+\sum_{e\in U}\Delta\mu_e\le b^{s+1}B_0 ,
\end{aligned}
\end{equation}
where $\Delta_e$ is the utility regained by lifting $e$ one rung, $c_e$ is its
regeneration cost, $n^{s}$ is the serving-token count of the stage, and $\rho$
is the compute cap (default $10\%$). The two constraints make the trade-off
explicit: promotion may spend at most a fixed share of serving compute and may
never exceed the byte budget. Greedy selection by density
$\tilde{v}_e\Delta_e/c_e$ gives the standard knapsack approximation for
\eqref{eq:promote} and runs in idle windows between queries.

A selected residue is rebuilt by conditioning the host model on $z_e$ and its
graph neighborhood, $e'\sim\mathcal{R}(\cdot\mid z_e,M^{s})$. The draft is
admitted only if the verification operator confirms it against the stored
invariants, $\mathcal{V}(e',h_e)=1$ (Fig.~\ref{fig:framework}(c)): every key
span, entity tuple, and outcome fact recorded at demotion time must be entailed
by $e'$. A rejected draft leaves
the entry at its current rung, so a failed promotion cannot degrade the store.
Hallucinated reconstructions are filtered before they enter retrieval. Admission
therefore converts serving tokens into recovered capability, the
storage-for-compute exchange the ladder was built to support.

Three properties locate the guarantees of the loop; only the last one is
idealized.

\begin{proposition}[Feasibility invariance]
\label{prop:feas}
Under Algorithm~\ref{alg:crystalmem}, $\mu(M^{s})\le b^{s}B_0$ at every
stage $s$.
\end{proposition}

\begin{proof}
The Crystallize loop demotes, and at the bottom rung evicts, until the
constraint holds. Each step strictly decreases $\mu(M)$, so the loop terminates
with a feasible store. Promotions are admitted only within the byte headroom in
\eqref{eq:promote}.
\end{proof}

\begin{proposition}[Near-optimality of the water-filling schedule]
\label{prop:greedy}
Suppose each entry's marginal price $\Delta\eta_e/\Delta\mu_e$ is
nondecreasing along the ladder. Then the Phase II schedule is optimal for
the fractional relaxation of \eqref{eq:schedule}, and its energy exceeds
the integral optimum by at most
$\max_e\left[\eta_e(K)-\eta_e(1)\right]$, the span of one entry.
\end{proposition}

\begin{proof}
Under nondecreasing prices, taking rung moves greedily by price matches the
fractional optimum up to the single boundary move where the byte constraint
becomes tight, by the classical fractional-knapsack argument. Rounding that
move integrally charges at most one entry's span, and the fractional optimum
lower-bounds the integral one.
\end{proof}

The promotion pass in \eqref{eq:promote} admits the same argument per binding
cap, losing at most one boundary entry for each cap.

\begin{proposition}[Loop closure under exact regeneration]
\label{prop:closure}
Suppose no bottom-rung eviction occurs, $\mathcal{R}$ restores each promoted
residue to its full utility, and $\mathcal{V}$ has no false rejections. If
the caps in \eqref{eq:promote} admit all candidates during the recovery
phase, then $\mathbb{E}\!\left[D_m\right]\to 0$ and the loop of
Definition~\ref{def:hysteresis} closes.
\end{proposition}

\begin{proof}
Every squeezed entry retains its residue. Under the stated hypotheses, each
entry is lifted back to $f_1$ as the budget returns, so the terminal store
matches the counterfactual store of Proposition~\ref{prop:bound} on the evicted
mass. The deficit term vanishes, the recovery branch retraces the squeeze
branch, and the enclosed area of Definition~\ref{def:hysteresis} is zero.
\end{proof}

Both greedy passes sort once and sweep once, $O(|M|\log|M|)$ per stage, with
no training and no external solver. The loop is cheap enough to run inside the
serving path.

\subsection{System Implementation}
\label{sec:impl}

\sys{} runs as a sidecar beside the serving stack~\cite{Kwon2023SOSP}. Monitor
hooks the retrieval path and updates \eqref{eq:influence} inline; Crystallize
fires on budget events; and Recrystallize batches its generation calls into
idle windows, so foreground queries do not wait on promotion. The store keeps,
for each entry, the current rung, the residue $z_e$, the invariants $h_e$, and
the graph edges of \eqref{eq:coupling}. Under multi-tenancy, the same schedule
runs over the pooled budget with tenant-weighted values and per-tenant
capability floors. On edge hosts, a small model runs Monitor and Crystallize
locally and delegates $\mathcal{R}$ to the hub.

\begin{algorithm}[t]
\caption{\sys{} Elastic-Stage Lifecycle
(\colorbox{phaseI}{Monitor}, \colorbox{phaseII}{Crystallize}, and
\colorbox{phaseIII}{Recrystallize})}
\label{alg:crystalmem}
\textbf{Input:} store $M^{s}$, incoming budget $b^{s+1}$, serving tokens
$n^{s}$, compute cap $\rho$, rates $\gamma,\lambda,\tau$
\begin{algorithmic}[1]
\STATE \textit{Phase I: Monitor} \Comment{budget and query telemetry}
\colorbox{phaseI}{
\parbox{0.82\columnwidth}{
\STATE on a budget event $b^{s+1}\neq b^{s}$, snapshot the loop point
$\left(b^{s}, C_m^{s}\right)$
\STATE on each retrieval, update the influence $I_e$ by \eqref{eq:influence}
\STATE refresh the coupled values $\tilde{v}_e$ by \eqref{eq:coupling}
}}

\STATE \textit{Phase II: Crystallize} \Comment{solve \eqref{eq:schedule}}
\colorbox{phaseII}{
\parbox{0.82\columnwidth}{
\WHILE{$\mu(M) > b^{s+1}B_0$}
\STATE pop the cheapest move
$e=\operatorname*{arg\,min}_{e}\,\Delta\eta_e/\Delta\mu_e$; raise the
multiplier $\theta$ to that price
\STATE demote $e$ one rung, distilling $(z_e,h_e)$ at its first demotion;
if already \textsc{trace}, evict $e$
\ENDWHILE
}}

\STATE \textit{Phase III: Recrystallize} \Comment{verified promotion}
\colorbox{phaseIII}{
\parbox{0.82\columnwidth}{
\STATE select the promotion set $U$ by \eqref{eq:promote}
\FORALL{$e\in U$}
\STATE draw $e'\sim\mathcal{R}(\cdot\mid z_e,M)$; admit one rung up iff
$\mathcal{V}(e',h_e)=1$
\STATE charge $c_e$ to the recrystallization ledger; stop once
$\rho\, n^{s}$ is spent
\ENDFOR
}}
\RETURN $M^{s+1}$
\end{algorithmic}
\end{algorithm}

Taken together, \sys{} turns retention into reversible state assignment.
Advantage-weighted influence and a dependency kernel value each entry; a
crystallization-energy schedule demotes entries along a four-state ladder to
satisfy the budget path; and a generation--verification operator pair buys
capability back under explicit compute and byte caps, admitting only
reconstructions that pass invariant checks. Deletion becomes the last resort
instead of the default.

\section{Experiments}
\label{sec:experiments}

\subsection{Experimental Setup}
\label{sec:setup}

\textit{1) Environments and Protocol:}
We evaluate on seven environments in four workload families: retrieval under
streaming writes (the SF and TTL tracks of~\cite{Hu2026ICLR}), long-context
dialogue (LoCoMo~\cite{Maharana2024ACL} and LongMemEval~\cite{Wu2025ICLR},
abbreviated LME), agentic control (ALFWorld~\cite{Shridhar2021ICLR} and
WebArena~\cite{Zhou2024ICLR}), and SynDrift, a synthetic protocol that
retires facts mid-stream (Drift). Every policy runs the elastic cycle of
Definition~\ref{def:schedule} under the same byte accounting. Capability is
measured at each of the seven stages. Unless a table states otherwise,
reported numbers aggregate five seeds; the cost frontier, the three
largest backbone scales, the multi-tenant study, and the deployment study
use three.

\textit{2) Comparison Methods:}
We compare \sys{} with fifteen baselines, grouped as in
Table~\ref{tab:main}. \textit{(i)}~Reference bounds: keep-all, which
ingests the same streams without a budget and bounds capability from above;
LRU; expiry; and random drop. \textit{(ii)}~Single-tier summarization:
Mem0~\cite{Chhikara2025ARXIV}, MemGPT~\cite{Packer2024COLM}, and
MaRS~\cite{Alqithami2025ARXIV}. \textit{(iii)}~Hierarchical and ladder
policies: A-MEM~\cite{Xu2025NEURIPS}, MemoryOS~\cite{Kang2025EMNLP}, the
Human-Inspired irreversible ladder~\cite{Kerestecioglu2026ARXIV}, and
R$^3$Mem~\cite{Wang2025ACL}, which stores reversible residues but promotes
them without a schedule. \textit{(iv)}~Value governance:
CURATOR~\cite{Wu2026ARXIV}, A-MAC~\cite{Zhang2026ICLRW},
RecMem~\cite{Dai2026ACL}, and AgeMem~\cite{Yu2026ARXIV} with its learned
admission policy frozen. Each baseline enforces the stage budget through
its native retention knob. \sys{} runs in two configurations throughout:
full, with dependency coupling on, and lite, with coupling off.

\textit{3) Evaluation Metrics:}
We report the three loop metrics of Definition~\ref{def:hysteresis}:
restored capability $C_m^S$, hysteresis area $\hat{H}_m$ estimated
by~\eqref{eq:est}, and residual deficit $D_m$. Capability is scored on a
$0$--$100$ scale, so deficits read directly in points (pp). Two meters
cover serving cost: the recrystallization ratio, defined as promotion
tokens divided by stream inference tokens, and the scheduler share of stage
wall-clock time.

\textit{4) Implementation Details:}
All policies serve from vLLM~\cite{Kwon2023SOSP} on A100 and H100 GPUs,
with a Qwen2.5-7B-Instruct host~\cite{Yang2024ARXIV} at temperature $0$
and a bge-m3 retriever~\cite{Chen2024ACL}. Five additional backbone scales
are evaluated in Sec.~\ref{sec:ablation}.

\subsection{Iso-Budget Main Comparison}
\label{sec:main}

\begin{table*}[t]
\centering
\caption{Iso-budget comparison over one elastic cycle (Definition~\ref{def:schedule}): restored capability $C_m^S$ (mean$\pm$std, five seeds), its seven-environment average, mean hysteresis area $\bar{H}$ \eqref{eq:est}, and mean residual deficit $\bar{D}$ (pp). The best value in each baseline family is \underline{underlined}; the gain of \sys{} over each row appears in the \colorbox{DeltaBg}{$\Delta$ column}. \textcolor{gray}{\textit{Gray italic}} is the unconstrained keep-all bound, not ranked.}
\label{tab:main}
\footnotesize
\setlength{\tabcolsep}{2.8pt}
\renewcommand{\arraystretch}{1.12}
\begin{tabular}{l|cc|cc|cc|c|ccc|c}
\toprule
& \multicolumn{2}{c|}{\cellcolor{SearchBg}\textbf{Retrieval}} & \multicolumn{2}{c|}{\cellcolor{MathBg}\textbf{Dialogue}} & \multicolumn{2}{c|}{\cellcolor{AgenticBg}\textbf{Agentic}} & \multicolumn{1}{c|}{\cellcolor{ScienceBg}\textbf{Synthetic}} & \multicolumn{3}{c|}{} & \\
\cmidrule(lr){2-3}\cmidrule(lr){4-5}\cmidrule(lr){6-7}\cmidrule(lr){8-8}
\textbf{Method} & \textbf{SF} & \textbf{TTL} & \textbf{LoCoMo} & \textbf{LME} & \textbf{ALF} & \textbf{Web} & \textbf{Drift} & \textbf{Avg}$\,\uparrow$ & $\bar{H}\,\downarrow$ & $\bar{D}\,\downarrow$ & \textbf{$\Delta$} \\
\midrule
{\color{gray}\textit{Keep-all}} & {\color{gray}\textit{66.8}\std{0.5}} & {\color{gray}\textit{59.1}\std{0.6}} & {\color{gray}\textit{58.3}\std{0.2}} & {\color{gray}\textit{55.5}\std{0.7}} & {\color{gray}\textit{68.9}\std{0.8}} & {\color{gray}\textit{30.3}\std{0.7}} & {\color{gray}\textit{72.5}\std{0.8}} & {\color{gray}\textit{58.8}} & --- & --- & {\color{gray}\textit{$-$1.3}} \\
LRU & 45.9\std{0.5} & 41.9\std{0.7} & 43.0\std{0.6} & 39.0\std{0.4} & 48.0\std{0.4} & 19.0\std{0.3} & 46.6\std{0.5} & 40.5 & 0.209 & 17.6 & \mygreen{68}\gain{16.9} \\
Expiry & \underline{49.3\std{0.7}} & \underline{44.8\std{0.1}} & \underline{43.3\std{0.7}} & \underline{40.8\std{0.2}} & \underline{51.5\std{0.8}} & \underline{19.8\std{0.8}} & \underline{51.2\std{0.4}} & \underline{43.0} & \underline{0.178} & \underline{15.1} & \mygreen{58}\gain{14.5} \\
Random drop & 46.2\std{0.5} & 41.7\std{0.6} & 40.4\std{0.4} & 38.4\std{0.3} & 48.4\std{0.7} & 17.9\std{0.4} & 47.4\std{0.2} & 40.0 & 0.216 & 17.9 & \mygreen{70}\gain{17.4} \\
\midrule
Mem0~\cite{Chhikara2025ARXIV}~\venue{arXiv'25} & 52.0\std{0.6} & \underline{47.5\std{0.6}} & 46.0\std{0.3} & 43.0\std{0.3} & 53.7\std{0.3} & 22.0\std{0.7} & 55.2\std{0.6} & 45.6 & 0.145 & 12.3 & \mygreen{47}\gain{11.8} \\
MemGPT~\cite{Packer2024COLM}~\venue{COLM'24} & \underline{52.8\std{0.5}} & 47.1\std{0.7} & \underline{46.4\std{1.1}} & \underline{43.7\std{1.1}} & \underline{54.8\std{0.9}} & \underline{22.0\std{0.8}} & \underline{55.4\std{0.6}} & \underline{46.0} & \underline{0.144} & \underline{12.1} & \mygreen{46}\gain{11.4} \\
MaRS~\cite{Alqithami2025ARXIV}~\venue{arXiv'25} & 51.4\std{0.7} & 46.5\std{0.6} & 45.0\std{0.5} & 42.2\std{1.0} & 53.3\std{0.3} & 21.1\std{0.3} & 53.6\std{0.6} & 44.7 & 0.157 & 13.3 & \mygreen{51}\gain{12.7} \\
\midrule
A-MEM~\cite{Xu2025NEURIPS}~\venue{NeurIPS'25} & 54.7\std{0.3} & 48.5\std{0.2} & 47.7\std{0.6} & 44.8\std{0.4} & 56.7\std{0.4} & 22.9\std{0.8} & 57.7\std{0.4} & 47.6 & 0.121 & 10.5 & \mygreen{39}\gain{9.9} \\
MemoryOS~\cite{Kang2025EMNLP}~\venue{EMNLP'25} & 56.0\std{0.3} & 49.7\std{0.4} & 48.9\std{0.2} & 45.7\std{0.2} & 58.0\std{1.4} & 23.9\std{0.3} & 59.2\std{0.7} & 48.8 & 0.109 & 9.2 & \mygreen{35}\gain{8.7} \\
Human-Inspired~\cite{Kerestecioglu2026ARXIV}~\venue{arXiv'26} & 54.6\std{0.4} & 48.7\std{0.5} & 48.0\std{0.8} & 45.0\std{0.4} & 57.1\std{0.5} & 22.6\std{0.5} & 58.1\std{0.4} & 47.7 & 0.122 & 10.1 & \mygreen{39}\gain{9.7} \\
R$^3$Mem~\cite{Wang2025ACL}~\venue{ACL'25} & \underline{60.0\std{0.6}} & \underline{53.5\std{0.4}} & \underline{53.1\std{0.6}} & \underline{49.6\std{0.3}} & \underline{62.5\std{0.8}} & \underline{26.3\std{0.8}} & \underline{64.6\std{0.3}} & \underline{52.8} & \underline{0.044} & \underline{5.2} & \mygreen{18}\gain{4.6} \\
\midrule
CURATOR~\cite{Wu2026ARXIV}~\venue{arXiv'26} & \underline{58.7\std{0.4}} & \underline{52.3\std{0.4}} & 51.1\std{0.6} & \underline{48.4\std{0.3}} & \underline{60.5\std{0.6}} & \underline{25.4\std{0.2}} & \underline{62.9\std{0.6}} & \underline{51.3} & \underline{0.087} & \underline{6.6} & \mygreen{24}\gain{6.1} \\
A-MAC~\cite{Zhang2026ICLRW}~\venue{ICLR-W'26} & 58.0\std{0.3} & 51.6\std{0.6} & \underline{51.2\std{0.6}} & 47.7\std{0.4} & 60.0\std{0.8} & 25.1\std{0.5} & 62.4\std{0.5} & 50.8 & 0.088 & 7.2 & \mygreen{26}\gain{6.6} \\
RecMem~\cite{Dai2026ACL}~\venue{ACL'26} & 57.6\std{0.7} & 51.3\std{0.4} & 50.4\std{0.5} & 47.6\std{0.1} & 59.5\std{0.4} & 24.9\std{0.8} & 61.7\std{0.3} & 50.4 & 0.093 & 7.6 & \mygreen{28}\gain{7.0} \\
AgeMem~\cite{Yu2026ARXIV}~\venue{arXiv'26} & 56.8\std{1.0} & 50.9\std{0.8} & 50.2\std{0.3} & 46.9\std{0.4} & 58.9\std{0.5} & 24.3\std{0.8} & 60.7\std{0.3} & 49.8 & 0.102 & 8.2 & \mygreen{30}\gain{7.6} \\
\midrule
\sys-lite & 65.0\std{0.6} & 56.7\std{0.4} & 56.4\std{0.9} & 52.7\std{0.7} & 67.4\std{0.6} & 29.0\std{0.3} & 70.9\std{1.2} & 56.9 & 0.013 & 1.1 & \mygreen{3}\gain{0.6} \\
\rowcolor{crystalblue!10}
\textbf{\sys{} (Ours)} & \textbf{65.1\std{0.4}} & \textbf{57.4\std{0.7}} & \textbf{56.6\std{0.5}} & \textbf{53.2\std{1.2}} & \textbf{68.6\std{0.5}} & \textbf{29.8\std{0.3}} & \textbf{71.3\std{1.0}} & \textbf{57.4} & \textbf{0.013} & \textbf{0.5} & -- \\
\bottomrule
\end{tabular}
\end{table*}

\begin{table}[t]
\centering
\caption{Iso-capability budget (three seeds). Ref is CURATOR, the strongest budgeted baseline at full provision; $b^{\star}$ is the smallest budget at which \sys{} matches it. Cells at or above Ref in \textbf{bold}; the median $b^{\star}$ in \best{red}.}
\label{tab:cost}
\footnotesize
\setlength{\tabcolsep}{5.5pt}
\renewcommand{\arraystretch}{1.12}
\begin{tabular}{l|c|ccc|c}
\toprule
& & \multicolumn{3}{c|}{\cellcolor{ScienceBg}\textbf{\sys{} at budget}} & \\
\cmidrule(lr){3-5}
\textbf{Env.} & \textbf{Ref@100} & \textbf{40\%} & \textbf{50\%} & \textbf{60\%} & \textbf{$b^{\star}$} \\
\midrule
SF & 66.3 & 66.2 & \textbf{67.0} & \textbf{67.1} & 50\% \\
TTL & 58.6 & \textbf{58.6} & \textbf{59.3} & \textbf{59.5} & 40\% \\
LoCoMo & 58.0 & 57.5 & \textbf{58.4} & \textbf{58.3} & 50\% \\
LME & 55.3 & 54.9 & \textbf{55.4} & \textbf{55.6} & 50\% \\
ALF & 68.5 & 68.1 & \textbf{68.7} & \textbf{69.1} & 50\% \\
Web & 29.9 & 29.3 & \textbf{30.6} & \textbf{30.7} & 50\% \\
Drift & 72.1 & 72.0 & \textbf{72.5} & \textbf{72.8} & 50\% \\
\midrule
\rowcolor{crystalblue!10}
\multicolumn{5}{l|}{\textbf{Median across environments}} & \best{50\%} \\
\bottomrule
\end{tabular}
\end{table}

\begin{figure*}[t]
\centering
\includegraphics[width=0.99\textwidth]{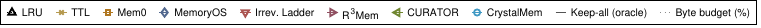}\\[-1.2em]
\subfloat[LoCoMo]{\includegraphics[width=0.245\textwidth]{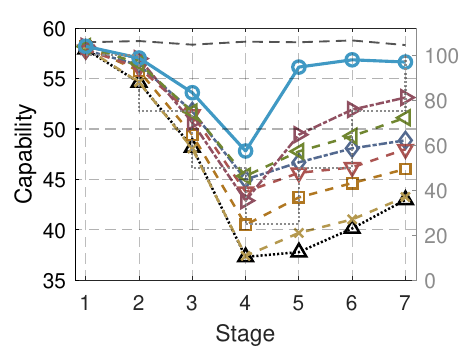}}\hfill
\subfloat[ALFWorld]{\includegraphics[width=0.245\textwidth]{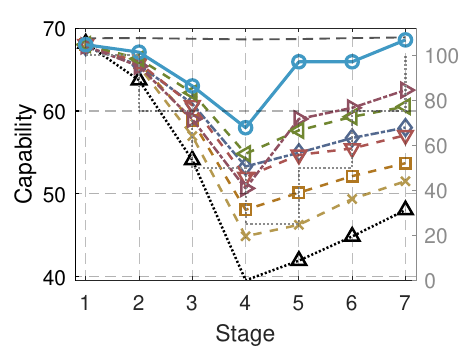}}\hfill
\subfloat[Loops (LoCoMo)]{\includegraphics[width=0.245\textwidth]{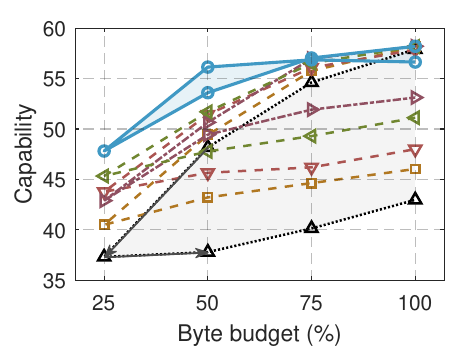}}\hfill
\subfloat[Recovered share]{\includegraphics[width=0.245\textwidth]{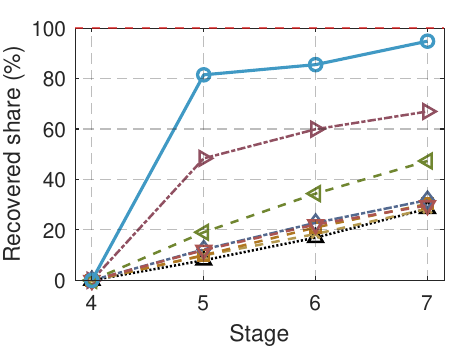}}
\caption{Elasticity dynamics under the squeeze--recover schedule (seed
means). (a),~(b)~Capability trajectories on LoCoMo and ALFWorld, with the
byte budget overlaid as gray stairs (right axis). (c)~Budget--capability
loops on LoCoMo: arrows give traversal direction, and shaded regions mark
the areas enclosed by LRU and by \sys. (d)~Share of trough loss recovered
at each recovery stage, averaged over all seven environments; the red
dashed line marks full recovery.}
\label{fig:dynamics}
\end{figure*}

Table~\ref{tab:main} gives the iso-budget comparison. \sys{} has the
highest restored capability on all seven environments, averaging $57.4$
against $52.8$ for R$^3$Mem, the strongest baseline in this comparison. The
margin averages $+4.6$\,pp and never falls below $+3.4$\,pp (WebArena). The loop
metrics separate the methods more sharply than raw capability: mean
hysteresis area falls to $0.013$ against a baseline range of
$0.044$--$0.216$, and mean residual deficit falls to $0.5$\,pp. The
remaining headroom is small, since keep-all, which never discards entries,
ends only $1.3$\,pp higher. The lite row isolates dependency coupling:
$0.6$\,pp on average, concentrated on ALFWorld ($1.1$\,pp) and WebArena
($0.8$\,pp), which Sec.~\ref{sec:ablation} revisits.

Figure~\ref{fig:dynamics} shows the cycles behind these aggregates. Every
method follows the squeeze downward. During recovery, the baselines flatten,
while \sys{} rejoins the keep-all trace, closes the budget--capability
loop, and recovers most of its trough loss one stage after the budget turns.

The deficit column shows the cost of each missing lever. R$^3$Mem keeps
residues at several fidelities but promotes them without an energy schedule
or a verification gate; reversibility without scheduling forfeits
$4.7$\,pp. The strongest binary policy, CURATOR, chooses evictions by
learned value but leaves no residue to rebuild from. The floor of
Proposition~\ref{prop:bound} then applies, and selection without
reversibility forfeits $6.1$\,pp. The Human-Inspired ladder keeps several
fidelities but moves entries only downward. At $10.1$\,pp, it trails even
binary governance, because demoted entries continue serving degraded
representations after the budget returns. Closing the loop needs both
levers: residues to rebuild from and a schedule that decides when to spend
them.

Two meters bound the serving cost of these gains. Recrystallization uses
$7.3\%$ of stream inference tokens on average and $9.1\%$ at worst, within
the token cap of~\eqref{eq:promote}. The scheduler adds $3.4\%$ of stage
wall-clock time on average and $3.8\%$ at most. Across the
$4{,}165$-measurement grid, no method exceeds its stage budget, and \sys{}
maintains the feasibility invariant of Proposition~\ref{prop:feas} at
every stage.

Table~\ref{tab:cost} restates the frontier in budget terms. The median
$b^{\star}$ across environments is $50\%$, with six environments matched
at half provision and TTL at $40\%$. Even a $40\%$ budget leaves at most a
$0.7$\,pp shortfall, and the curve is flat above the match point:
capability at a $60\%$ budget stays within $0.3$\,pp of full provision on
every environment. Over the eight-budget grid, \sys{} is never below LRU,
Mem0, or CURATOR at any operating point, so the ranking does not depend on
where a provider sits on the frontier. The provisioning rule is simple:
run the memory plane at half its nominal budget and lose nothing relative
to the best fixed-provision baseline. Figure~\ref{fig:frontier} draws the
frontier: \sys{} meets the keep-all reference at half budget, and the
per-environment curves keep it above the strongest binary baseline at
every operating point.

\begin{figure}[t]
\centering
\subfloat[Frontier]{\includegraphics[width=0.49\columnwidth]{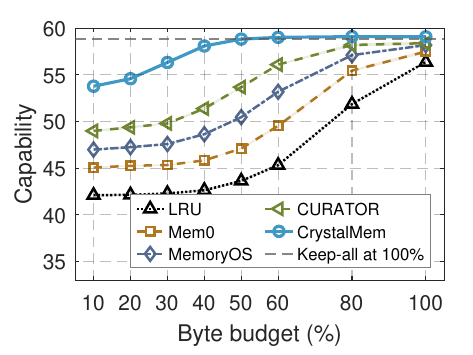}}\hfill
\subfloat[Per environment]{\includegraphics[width=0.49\columnwidth]{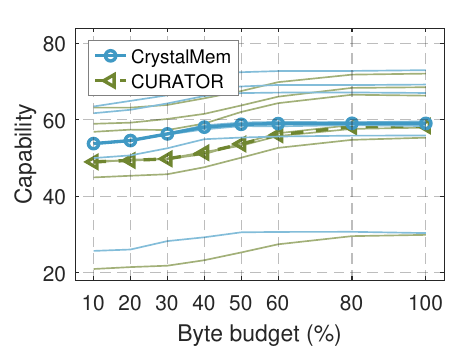}}
\caption{Cost frontier under constant provisioning (seed means).
(a)~Seven-environment mean capability against the byte budget; the dashed
line is keep-all at full provision; shaded bands give seed variation.
(b)~Per-environment curves (thin) and their means (thick) for \sys{} and
the strongest binary baseline.}
\label{fig:frontier}
\end{figure}

\subsection{Mechanism and Theoretical Validation}
\label{sec:mechanism}

\begin{figure*}[t]
\centering
\subfloat[Fidelity composition]{\includegraphics[width=0.245\textwidth]{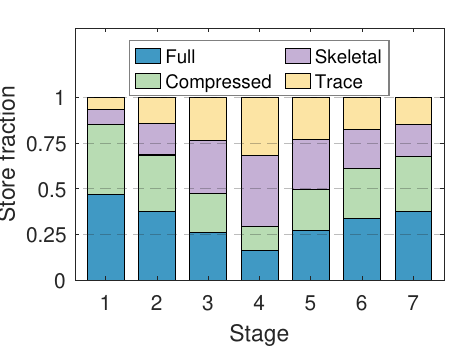}}\hfill
\subfloat[Influence estimator]{\includegraphics[width=0.245\textwidth]{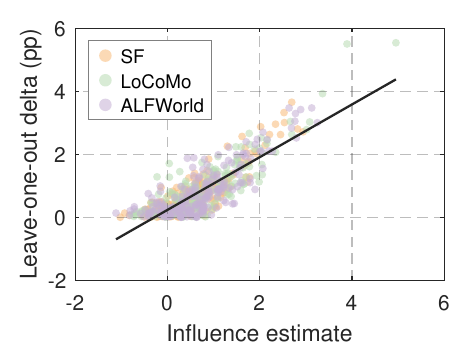}}\hfill
\subfloat[Verification gate]{\includegraphics[width=0.245\textwidth]{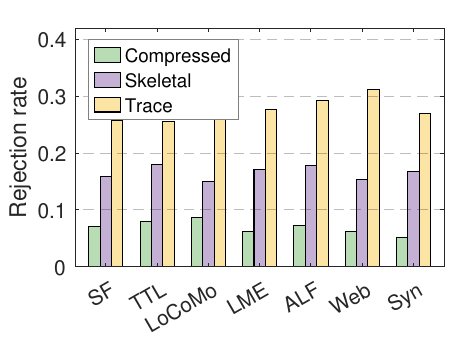}}\hfill
\subfloat[Deficit bound]{\includegraphics[width=0.245\textwidth]{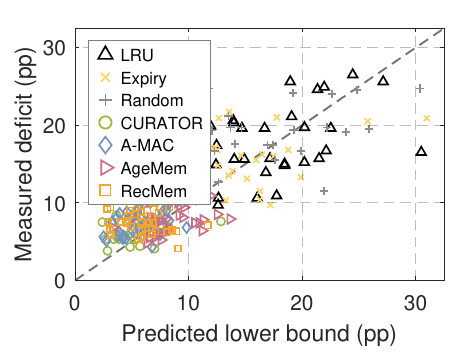}}
\caption{Mechanism validation on live cycles (seed means). (a)~Fidelity
composition of the store per stage. (b)~Influence estimate against
leave-one-out ground truth, with a linear fit. (c)~Verification rejection
rate by origin state and environment. (d)~Measured deficit against the
estimated bound of Proposition~\ref{prop:bound} for the seven binary
policies; points above the gray diagonal satisfy the bound.}
\label{fig:mech}
\end{figure*}

Figure~\ref{fig:mech} traces the machinery of
Algorithm~\ref{alg:crystalmem} on live cycles. Panel~(a) tracks the
store's fidelity composition. At the start, $85\%$ of entries occupy the
upper two states, \textsc{full} and \textsc{compressed}. The squeeze pushes
all but $29\%$ down to \textsc{skeletal} and \textsc{trace}; verified
promotion raises the upper share back to $68\%$ by the final stage. The
store moves in both directions through fidelity space.

Panel~(b) checks the influence estimator \eqref{eq:influence} against
leave-one-out ground truth on $200$ entries in each of three environments.
Spearman correlation reaches $0.75$--$0.76$, and $75$--$80\%$ of the true
top-decile entries land in the estimated top decile. The
advantage-weighted signal is accurate enough to order demotions.

Panel~(c) splits recrystallization by origin state. The gate rejects
$6.9\%$ of reconstructions from \textsc{compressed}, $16.4\%$ from
\textsc{skeletal}, and $27.5\%$ from \textsc{trace}. Rejection rises with
demotion depth, as sparser residues leave more room for invariant
violations. Admitted reconstructions recover about $90\%$ of retrieval
utility regardless of origin, and their downstream contribution remains
positive ($+0.9$ to $+1.0$\,pp). Sec.~\ref{sec:ablation} shows the cost of
removing the gate.

Panel~(d) tests Proposition~\ref{prop:bound} where it applies. Across the
seven binary policies, measured deficits track the estimated bound with
correlation $0.76$ and exceed it on two thirds of the $245$
environment--seed points. \sys{} breaks
the premise rather than the bound: it keeps residues, and its deficit
($0.5$\,pp, Table~\ref{tab:main}) sits an order of magnitude below every
binary floor. Together with the zero budget violations of
Sec.~\ref{sec:main} and the closed loops of Table~\ref{tab:main}, these
measurements ground Propositions~\ref{prop:feas} and~\ref{prop:closure} in
the observed cycles.

\subsection{Ablations and Sensitivity}
\label{sec:ablation}

\begin{table}[t]
\centering
\caption{Ablation over one elastic cycle (seven environments, five seeds): seven-environment averages, with $\Delta$ the drop in Avg relative to full \sys{} (default: water-filling schedule, advantage-weighted influence, $\rho=10\%$, $K{=}4$ states).}
\label{tab:ablation}
\footnotesize
\setlength{\tabcolsep}{4.5pt}
\renewcommand{\arraystretch}{1.15}
\begin{tabular}{l|ccc|c}
\toprule
\textbf{Variant} & \textbf{Avg}$\,\uparrow$ & $\bar{H}\,\downarrow$ & $\bar{D}\,\downarrow$ & \textbf{$\Delta$} \\
\midrule
\rowcolor{crystalblue!18}
\sys{} (full) & \best{57.3\std{0.3}} & 0.015 & 0.7 & -- \\
\midrule
Schedule $\to$ size-greedy & 55.6\std{0.2} & 0.015 & 2.2 & \mygreen{7}\gain{1.8} \\
\rowcolor{crystalblue!5}
Schedule $\to$ random & 54.3\std{0.2} & 0.033 & 3.7 & \mygreen{12}\gain{3.0} \\
\midrule
Influence $\to$ recency & 56.0\std{0.2} & 0.012 & 1.9 & \mygreen{5}\gain{1.3} \\
\rowcolor{crystalblue!5}
Influence $\to$ random & 53.7\std{0.2} & 0.035 & 4.3 & \mygreen{15}\gain{3.7} \\
\midrule
w/o dependency coupling & 56.9\std{0.2} & 0.011 & 1.0 & \mygreen{3}\gain{0.4} \\
\rowcolor{crystalblue!5}
w/o verification gate & 54.6\std{0.2} & 0.025 & 3.3 & \mygreen{11}\gain{2.8} \\
\midrule
Recryst.\ cap $0\%$ & 51.4\std{0.2} & 0.071 & 6.5 & \mygreen{24}\gain{5.9} \\
\rowcolor{crystalblue!5}
Recryst.\ cap $5\%$ & 55.9\std{0.2} & 0.010 & 2.1 & \mygreen{6}\gain{1.4} \\
Recryst.\ cap $20\%$ & 57.1\std{0.1} & 0.013 & 0.9 & \mygreen{3}\gain{0.3} \\
\midrule
Ladder $K{=}2$ & 54.3\std{0.2} & 0.024 & 3.5 & \mygreen{12}\gain{3.0} \\
\rowcolor{crystalblue!5}
Ladder $K{=}3$ & 56.6\std{0.3} & 0.011 & 1.5 & \mygreen{3}\gain{0.8} \\
Ladder $K{=}6$ & 57.2\std{0.1} & 0.015 & 0.7 & \mygreen{3}\gain{0.1} \\
\bottomrule
\end{tabular}
\end{table}

\begin{table}[t]
\centering
\caption{Backbone sweep (averages over SF, ALFWorld, and SynDrift; five seeds up to 14B, three above): restored capability and hysteresis area per model scale. $\Delta$ is the \sys{} gain over CURATOR at the same backbone.}
\label{tab:backbones}
\footnotesize
\setlength{\tabcolsep}{3.6pt}
\renewcommand{\arraystretch}{1.15}
\begin{tabular}{l|cc|cc|cc|c}
\toprule
& \multicolumn{2}{c|}{\textbf{LRU}} & \multicolumn{2}{c|}{\textbf{CURATOR}} & \multicolumn{2}{c|}{\textbf{\sys}} & \\
\cmidrule(lr){2-3}\cmidrule(lr){4-5}\cmidrule(lr){6-7}
\textbf{Backbone} & \textbf{Avg} & $\bar{H}$ & \textbf{Avg} & $\bar{H}$ & \textbf{Avg} & $\bar{H}$ & \textbf{$\Delta$} \\
\midrule
Qwen2.5-3B & 39.0 & 0.240 & 52.6 & 0.093 & \textbf{60.3} & \textbf{0.012} & \mygreen{31}\gain{7.6} \\
Qwen2.5-7B & 46.9 & 0.213 & 60.7 & 0.080 & \textbf{68.4} & \textbf{0.014} & \mygreen{31}\gain{7.7} \\
Qwen2.5-14B & 50.0 & 0.203 & 63.8 & 0.079 & \textbf{71.3} & \textbf{0.013} & \mygreen{30}\gain{7.5} \\
Qwen2.5-32B & 51.7 & 0.196 & 65.9 & 0.076 & \textbf{73.2} & \textbf{0.010} & \mygreen{29}\gain{7.3} \\
Qwen2.5-72B & 53.8 & 0.195 & 67.7 & 0.075 & \textbf{75.6} & \textbf{0.011} & \mygreen{32}\gain{7.9} \\
Llama-3.3-70B & 52.9 & 0.197 & 66.8 & 0.076 & \textbf{74.1} & \textbf{0.012} & \mygreen{29}\gain{7.3} \\
\bottomrule
\end{tabular}
\end{table}

\begin{figure}[t]
\centering
\subfloat[Compute cap]{\includegraphics[width=0.49\columnwidth]{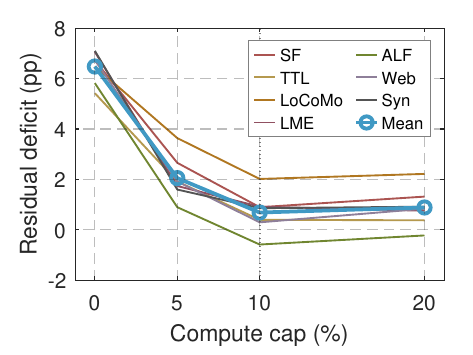}}\hfill
\subfloat[Ladder depth]{\includegraphics[width=0.49\columnwidth]{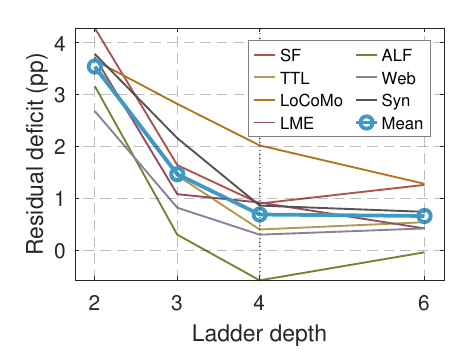}}
\caption{Design-parameter sweeps (seed means). Thin curves: single
environments; thick: the seven-environment mean; dotted vertical: the
deployed default.}
\label{fig:sens}
\end{figure}

Table~\ref{tab:ablation} removes one design choice at a time; every variant
loses capability, so no component is redundant. The full row re-runs the
Table~\ref{tab:main} configuration, and its average lands within
$0.1$\,pp. Replacing
the water-filling schedule \eqref{eq:schedule} with size-greedy demotion
costs $1.8$\,pp, and random demotion costs $3.0$\,pp: which bytes are
reclaimed matters as much as how many. The demotion ordering relies on the
influence signal~\eqref{eq:influence}; recency costs $1.3$\,pp, and random
scoring costs $3.7$\,pp. Removing the verification gate costs $2.8$\,pp
even though every reconstruction is admitted, because unverified
regenerations pollute retrieval. This is the failure mode rejected by the
gate in Fig.~\ref{fig:mech}(c). Dependency coupling is the smallest term,
$0.4$\,pp on the seven-environment average, concentrated on the two
agentic environments, matching the lite gap in Table~\ref{tab:main}.

The recrystallization cap prices loop closure directly. At $\rho=0$, the
deficit regresses to $6.5$\,pp, within noise of the strongest binary floor
($6.6$\,pp, Table~\ref{tab:main}); a reversible ladder without promotion
compute is a slower way to forget. At $\rho=5\%$, the deficit falls to
$2.1$\,pp, and raising $\rho$ from the default $10\%$ to $20\%$ buys only
$0.3$\,pp. Ladder depth shows the same knee: two states forfeit
$3.0$\,pp, three forfeit $0.8$\,pp, and six change nothing ($0.1$\,pp), so
the four-state ladder sits at the point of vanishing marginal return.
Figure~\ref{fig:sens} plots both sweeps per environment. The knee at the
deployed cap and the plateau past four states appear in every environment,
not only in the mean.

Table~\ref{tab:backbones} repeats the cycle on six backbones spanning
Qwen2.5~\cite{Yang2024ARXIV} from 3B to 72B and
Llama-3.3-70B~\cite{Grattafiori2024ARXIV}, on the SF, ALFWorld, and
SynDrift subset. Hysteresis is not a 7B artifact: LRU's area stays between
$0.195$ and $0.240$ at every scale. Scale also does not close the loop by
itself. The strongest binary baseline still forfeits $7.5$--$7.9$\,pp on
every backbone. \sys{} holds its area at or below $0.014$ and its deficit
at or below $0.5$\,pp on all six scales, with a steady $+7.3$ to
$+7.9$\,pp margin over it. The 3B row supports the deployment premise of
Sec.~\ref{sec:deployment}: the margin survives at edge scale.

The design does not depend on the seven environments used for tuning. With
every hyper-parameter frozen, one cycle on
BabyAI~\cite{ChevalierBoisvert2019ICLR}, an agentic stream outside that
set, leaves \sys{} with a $1.5$\,pp residual deficit and a $0.011$ loop
area, against $5.5$\,pp and $0.069$ for the strongest binary baseline and
$17.2$\,pp and $0.199$ for recency eviction. The family ordering is
unchanged.

\subsection{Fidelity-SLO Multi-Tenant Crystallization}
\label{sec:multitenant}

\begin{table}[t]
\centering
\caption{Multi-tenant crystallization (three seeds): share of windows that violate the per-tenant floor, by tenant count. $\Delta$ is the reduction \sys-SLO delivers at $N{=}8$, computed on unrounded rates; ms is the allocator decision time at $N{=}4$. Pooled Jain fairness is at least $0.999$ for every policy and tenant count.}
\label{tab:tenant}
\footnotesize
\setlength{\tabcolsep}{4.8pt}
\renewcommand{\arraystretch}{1.15}
\begin{tabular}{l|ccc|c|c}
\toprule
& \multicolumn{3}{c|}{\textbf{SLO violations}$\,\downarrow$} & & \\
\cmidrule(lr){2-4}
\textbf{Policy} & $N{=}2$ & $N{=}4$ & $N{=}8$ & \textbf{$\Delta$} & \textbf{ms} \\
\midrule
LRU-fair & 31\% & 38\% & 62\% & \mygreen{83}\gain{56} & 3 \\
CURATOR-fair & 17\% & 31\% & 47\% & \mygreen{60}\gain{40} & 4 \\
\sys-fair & 8\% & 14\% & 26\% & \mygreen{30}\gain{20} & 8 \\
\rowcolor{crystalblue!10}
\textbf{\sys-SLO (Ours)} & \textbf{3\%} & \textbf{3\%} & \textbf{6\%} & -- & \textbf{24} \\
\bottomrule
\end{tabular}
\end{table}

We next share one memory pool among $N$ tenants with independent streams
drawn from the seven environments, on a single vLLM instance. The pool is
provisioned at $N\times 40\%$ of $B_0$, squeezed to half for sustained
windows, and then released, with six windows per run. A tenant's service
level objective (SLO) holds in a window if its capability stays within
$2$\,pp of its solo capability at the same budget. Table~\ref{tab:tenant}
compares fair splitting of each store with the \sys{} allocator, which
steers the pool toward tenants nearest their floor.

Crystallization improves sharing before any allocation intelligence is
added. Fair splitting of a crystallized store violates $13.9\%$ of windows
at $N{=}4$, compared with $30.6\%$ for CURATOR and $37.5\%$ for LRU,
because a squeezed tenant keeps residues to recover from. Every policy in
this experiment is byte-fair (Jain $\ge 0.999$); the violations come from
what the bytes can rebuild, not from how they are split. The SLO-aware
allocator drops violations to $2.8\%$, with decisions in $24$\,ms on
average and $42$\,ms maximum, well under one percent of a window. Density
widens the separation: the gap over CURATOR-fair grows from $14$ to $28$
to $40$\,pp as $N$ goes $2\to 8$, and at $N{=}8$ recency eviction violates
$62\%$ of windows. Sharing at density is survivable only for stores that
can rebuild what the pool takes away.

\subsection{Real-System Deployment}
\label{sec:deployment}

\begin{table}[t]
\centering
\caption{Edge deployment on the Jetson testbed (two Orin agents, ALFWorld and SynDrift, three seeds): residual deficit (pp), recrystallization latency p50/p95 (ms), per-stage edge--cloud synchronization (MB), foreground blocking, and per-stage energy (J).}
\label{tab:edge}
\footnotesize
\setlength{\tabcolsep}{3.4pt}
\renewcommand{\arraystretch}{1.15}
\begin{tabular}{l|c|c|c|c|c}
\toprule
\textbf{Policy} & $\bar{D}\,\downarrow$ & \textbf{p50/p95} & \textbf{Sync}$\,\downarrow$ & \textbf{Block}$\,\downarrow$ & \textbf{Energy} \\
\midrule
LRU & 9.1 & -- & 16.0 & 0.2\% & 491 \\
CURATOR & 6.1 & -- & 16.9 & 0.2\% & 463 \\
\rowcolor{crystalblue!10}
\textbf{\sys{} (Ours)} & \textbf{2.0} & \textbf{586\,/\,2007} & \textbf{3.4} & \textbf{2.9\%} & \textbf{479} \\
\bottomrule
\end{tabular}
\end{table}

\begin{figure}[t]
\centering
\includegraphics[width=0.9\columnwidth]{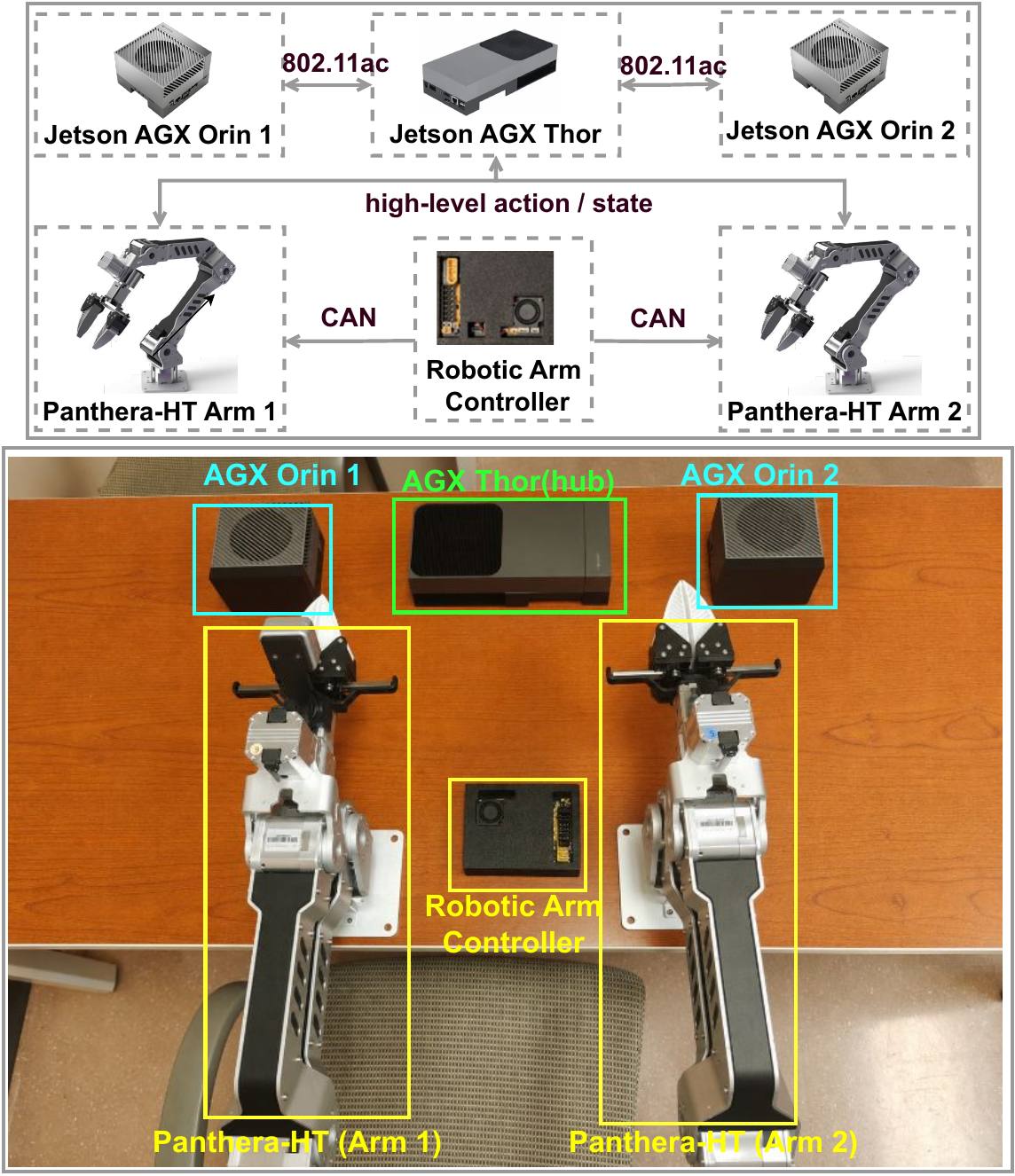}
\caption{Physical testbed: topology (top) and deployed hardware (bottom).}
\label{fig:testbed}
\end{figure}

\begin{figure*}[t]
\centering
\subfloat[On-device trajectories]{\includegraphics[width=0.245\textwidth]{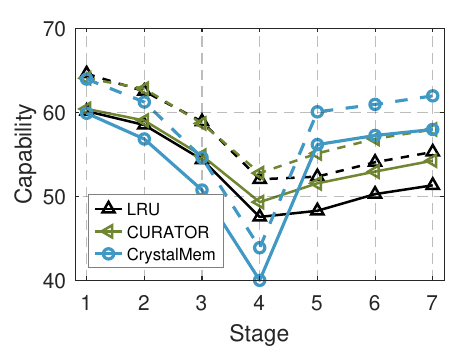}}\hfill
\subfloat[Recrystallization latency]{\includegraphics[width=0.245\textwidth]{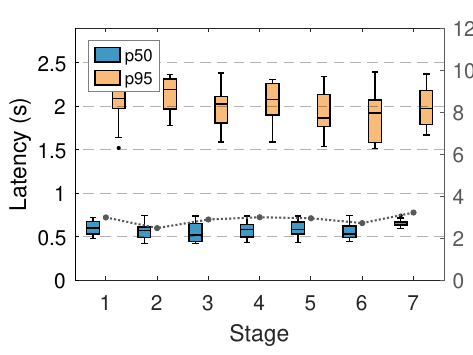}}\hfill
\subfloat[Synchronization traffic]{\includegraphics[width=0.245\textwidth]{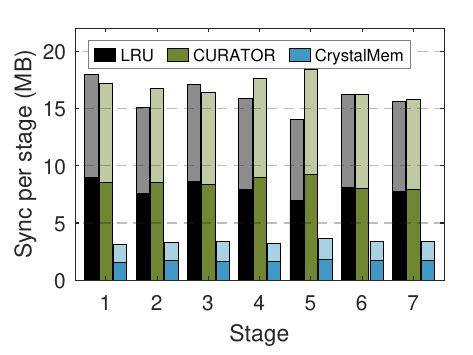}}\hfill
\subfloat[Recovery energy]{\includegraphics[width=0.245\textwidth]{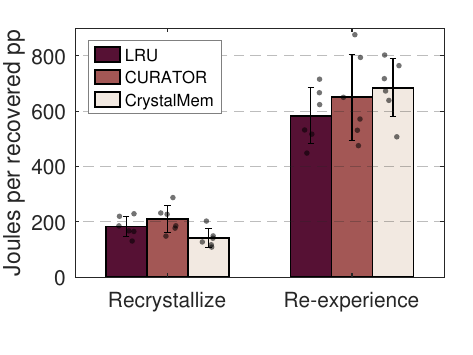}}
\caption{Edge deployment profile (seed means). (a)~Capability on the Orin
agents: solid ALFWorld, dashed SynDrift. (b)~\sys{} recrystallization
latency per stage (p50 and p95 boxes over devices, environments, and
seeds) with the foreground-blocking share on the right axis (dotted).
(c)~Per-stage synchronization traffic; the darker segment is uplink.
(d)~Energy per recovered capability point: recrystallization against
re-experience; dots are environment--seed observations.}
\label{fig:edge}
\end{figure*}

The physical testbed has a Jetson AGX Thor 128\,GB hub serving Qwen2.5-7B
through vLLM and two AGX Orin 64\,GB agents whose memory planes run
on-device Qwen2.5-3B, connected over real 802.11ac
(Fig.~\ref{fig:testbed}). Each Orin also drives a HighTorque Panthera-HT
arm, whose low-level motion runs on a shared CAN controller. The elastic
budget is the edge storage quota,
cycled through Definition~\ref{def:schedule} on ALFWorld and SynDrift with
three seeds; power is metered at $10$\,Hz through INA3221.

The capability results survive on hardware (Table~\ref{tab:edge}).
Figure~\ref{fig:edge} profiles the cycle: \sys{} dips deepest at the
trough but rejoins one stage after the turn, with flat latency and blocking
across stages. \sys{} restores to within $2.0$\,pp, compared with
$6.1$\,pp for CURATOR and $9.1$\,pp for LRU, preserving the cluster
ordering. Its costs remain serving-compatible: recrystallization runs at
$586$\,ms median and $2.0$\,s p95 per entry batch on the Orin, blocks
foreground inference $2.9\%$ of the time, and crystallized synchronization
moves $3.4$\,MB per stage against $16$--$17$\,MB for full-fidelity
shipping, a $4.8$--$5.0\times$ reduction over the real link. Per-stage
energy on the Orin agents is flat across policies ($463$--$491$\,J), so
loop closure adds no energy premium at the edge. The dedicated recovery
comparison sharpens the energy
account: recovering one point of capability by recrystallization costs
$141$\,J, compared with $684$\,J by re-experiencing the underlying
episodes, a $0.21\times$ ratio. A four-process tenant mini-deployment
across the three Jetsons ran the same cycles concurrently with no
out-of-memory events, per-process capability within $0.6$\,pp of solo runs,
and hub p95 latency at $1.5\times$ solo.

As a closing case study, a HighTorque arm driven by an Orin-hosted skill
memory runs one squeeze--recover cycle: the edge quota falls to $25\%$ and
returns. Under \sys{}, the arm's skill fragments demote to residues and
recrystallize on recovery; post-recovery task success is $89\%$ of the
pre-squeeze level. Under binary eviction, the demoted skills are gone, and
success plateaus at $61\%$.

\section{Conclusion}
\label{sec:conclusion}

In cloud-served self-evolving large language model agents, we identified
\emph{memory hysteresis}: after an experience-store budget is squeezed and
restored, capability settles below the pre-squeeze level. Deletion and
one-way compression cause the gap because they discard the material needed
for later rebuilding; we also prove a residual-deficit floor for any
keep/drop-only policy. \sys{} closes the loop by replacing binary memory
management with crystallized memory: entries move across four fidelity
states under a crystallization-energy schedule, demotions are ordered by
advantage-weighted influence with dependency coupling, and verified
recrystallization recovers capability under explicit compute and byte caps.
Across seven environments, seventeen methods, and six backbones, \sys{}
achieves the highest restored capability everywhere, leaves a residual
deficit an order of magnitude below every binary baseline, and, from half
the byte budget, matches the strongest budgeted baseline at full
provision. The advantage
carries into deployment: at the four-tenant operating point, crystallized
stores yield an order of magnitude fewer service-level violations than fair
splitting of binary stores, and on a physical edge--cloud testbed, each
recovered point of capability costs about a fifth of the energy required to
re-experience the underlying episodes.

\bibliographystyle{IEEEtran}
\bibliography{bib/refs}

\end{document}